\documentclass[a4paper]{article}

\usepackage[T1]{fontenc}

\usepackage{amsfonts,amsmath,amsthm}
\usepackage{xspace,hyperref,booktabs,tikz}
\usepackage{authblk}

\newtheorem{theorem}             {Theorem}
\newtheorem{lemma}      [theorem]{Lemma}
\newtheorem{corollary}  [theorem]{Corollary}

\newtheorem{definition} [theorem]{Definition}

\newcommand{\prob}[1]{\Pr\left(#1\right)}
\newcommand{\expect}[1]{\mathbf{E}\left[#1\right]}

\newcommand{\Real}{\mathbb{R}}

\newcommand{\opt}[1]{\ensuremath{\mathrm{OPT}_{#1}}}

\newcommand{\psel}{\ensuremath{p_\mathrm{sel}}\xspace}
\newcommand{\pmut}{\ensuremath{p_\mathrm{mut}}\xspace}
\newcommand{\pmutea}{\ensuremath{p_\mathrm{mut}^\mathrm{EA}}\xspace}

\newcommand{\hball}{\ensuremath{\mathrm{B}}}
\newcommand{\hdist}{\ensuremath{\mathrm{H}}}
\newcommand{\hbfunc}{\ensuremath{\mathrm{HB}}}
\newcommand{\mhbfunc}{\ensuremath{\mathrm{MHB}}}

\DeclareMathOperator{\poly}{poly}
\DeclareMathOperator{\bin}{Bin}
\DeclareMathOperator{\unif}{Unif}
\DeclareMathOperator{\pois}{Pois}

\DeclareMathOperator{\hgeo}{Hypergeo}

\newcommand{\hamming}[2]{\mathord{\hdist}\mathord{\left(#1, #2\right)}} 
\newcommand{\Prob}[1]{\prob{#1}} 
\newcommand{\bigO}[1]{\mathord{\text{O}}\mathord{\left(#1\right)}}
\newcommand{\littleo}[1]{\mathord{\text{o}}\mathord{\left(#1\right)}}
\newcommand{\bigTheta}[1]{\mathord{\Theta}\mathord{\left(#1\right)}}
\newcommand{\bigOmega}[1]{\mathord{\Omega}\mathord{\left(#1\right)}}
\newcommand{\littleomega}[1]{\mathord{\omega}\mathord{\left(#1\right)}}

\usepackage{mathrsfs,dsfont}

\newcommand{\indf}[1]{\ensuremath{\mathds{1}_{\{#1\}}}}
\newcommand{\indfs}[1]{\ensuremath{\mathds{1}_{#1}}}

\newcommand{\ab}{\hspace{0.125em}}                        \newcommand{\ie}{\hbox{i.\ab e.}\xspace}                  \newcommand{\eg}{\hbox{e.\ab g.}\xspace}                  

\usepackage{algorithm,algorithmic}

\allowdisplaybreaks

\usepackage{enumitem}

\usepackage{color}

\begin{document}

\title{\Huge Populations can be essential in tracking dynamic optima}

\author[1]{Duc-Cuong Dang}
\author[2]{Thomas Jansen}
\author[1]{Per Kristian Lehre}
\affil[1]{School of Computer Science, University of Nottingham \protect\\
          Jubilee Campus, 
          NG8 1BB, Nottingham, UK \protect\\
          {\normalsize\tt \{duc-cuong.dang,PerKristian.Lehre\}@nottingham.ac.uk}}
\affil[2]{Department of Computer Science, Aberystwyth University \protect\\
          Penglais Campus, 
          SY23 3DB, Aberystwyth, UK \protect\\
          {\normalsize\tt t.jansen@aber.ac.uk}}

\date{}

\maketitle

\makeatletter{}
\begin{abstract}
Real-world optimisation problems are often dynamic. Previously good solutions 
must be updated or replaced due to changes in objectives and constraints. It is 
often claimed that evolutionary algorithms are particularly suitable for dynamic 
optimisation because a large population can contain different 
solutions that may be useful in the future. However, rigorous theoretical 
demonstrations for how populations in dynamic optimisation can be essential are 
sparse and restricted to special cases.

This paper provides theoretical explanations of how populations can be essential 
in evolutionary dynamic optimisation in a general and natural setting. 
We describe a natural class of dynamic optimisation problems where a 
sufficiently large population is necessary to keep track of moving optima 
reliably. We establish a relationship between the population-size and the 
probability that the algorithm loses track of the optimum.
\end{abstract}


\makeatletter{}
\section{Introduction}\label{sec:intro}

In a classical optimisation setting, so-called \emph{static
  optimisation}, the focus is usually directed to finding an optimal
or a high quality solution as fast as possible. In real-world
optimisation, problem specific data may change over time,
thus previously good solutions can lose their quality and must be
updated or replaced. Automatic optimal control is a typical
illustration of these situations, e.\,g.\ parameters of a
machine can be set optimally under ideal conditions of a factory but
they need to be adapted to changes in the real environment upon
deployment. \emph{Dynamic optimisation} is an area of research that is
concerned with such optimisation problems that change over time. A
specific characteristic is that it does not only focus on
locating an optimal solution but also on tracking a moving optimum
(see \cite{bib:Fu2014} for a definition).

It is often suggested that Evolutionary Algorithms (EAs), especially
the ones with populations, are suitable for dynamic optimisation
because a large population can contain different solutions which could
be useful in the future \cite{bib:Yang2013}. However, 
theoretical demonstrations for how populations in dynamic optimisation
can be essential are sparse and restricted to special cases.  The
ability of a very simple EA without a population, the $(1+1)$~EA, to
track a target bitstring in a {\sc OneMax}-like function is analysed
in \cite{bib:Droste2003,bib:Stanhope1999}. The analysis has recently
been extended from bitstrings to larger alphabets
\cite{KoetzingLissovoiWitt2015}.  The influence of magnitude and
frequency of changes on the efficiency of the $(1+1)$~EA in optimising
a specifically designed function was investigated in
\cite{bib:Rohlfshagen2009}, showing that some dynamic optimisation
problems become easier with higher frequency of change. The analysis
of the $(1+\lambda)$~EA that uses a larger offspring population but
still not a real population on a simple lattice problem is presented
in \cite{bib:Jansen2005}. The efficiencies of specific diversity
mechanisms when using an actual population were compared in
\cite{bib:Oliveto2013}. This was done for a specific example function
(introduced by \cite{bib:Rohlfshagen2009}) and considering low
frequency of changes. It was shown in \cite{bib:Kotzing2012} that a
Min-Max Ant System (MMAS) can beat the $(1+1)$~EA in a deterministic
dynamic environment. The comparison was later extended to general
alphabets and to the $(\mu+1)$~EA that preserves genotype diversity
\cite{bib:Lissovoi2015b}. With that particular setting of the
$(\mu+1)$~EA, the size of the alphabets defines a threshold on the
parent population size $\mu$ so that the algorithm is able to track
and reach the optimal solution in polynomial time. The result was also
extended to the single-destination Shortest Path Problem
\cite{bib:Lissovoi2015}.  Comparisons were also made between EAs and
Artificial Immune System (AIS) on a {\sc OneMax}-like problem with the
dynamic being periodic \cite{bib:Jansen2014}.

Considering the existing analyses we can in summary note two shortcomings 
that leave the impression that important fundamental questions about dynamic 
optimisation are still not answered satisfactorily. One shortcoming is 
the concentration on simple evolutionary algorithms and other search 
heuristics that do not make use of an actual population. Clearly, the advantages 
of a population-based approach cannot be explored and explained this way. The 
other is that many studies consider very complex dynamic environments that make 
it hard to see the principal and fundamental issues. Therefore, the fundamental 
question why even a simple population without complicated diversity mechanisms
can be helpful in dynamic environments requires more attention. 

Motivated by the above facts, we will use a simple argument considering 
a very general class of dynamic functions to show that a population is 
essential to keep track of the optimal region. We define our function class
on the most often used search space, bit strings of a fixed length. However, it
is not difficult to extend the function class to 
be defined for any finite search space $\mathcal{X}$ and any 
\emph{unary} mutation operator $\pmut\colon \mathcal{X} \rightarrow \mathcal{X}$.
The class is called \emph{$(cn,\rho)$-stable} on 
$\mathcal{X}$ with respect to $\pmut$, where $n$ is the required number of bits 
to specify a search point of $\mathcal{X}$ and $c$ and $\rho$ are positive 
constants independent of $n$. The function class is
only restricted by the probability of recovering the optimal region 
via the mutation operator \pmut (see Definition~\ref{def:stable-dyn-func}).
The definition of the function class does not refer explicitly to other
function characteristics, such as the topology or the fitness
values of the set of optimal solutions, or the distribution of fitness
values of the set of non-optimal solutions.

We will use the \emph{Moving Hamming 
Ball} function from \cite{bib:Dang2015} as an illustrative example over the 
search space $\{0,1\}^{n}$ and with respect to the bitwise mutation operator. We 
also use this specific function to argue that an approach based on a single individual, such 
as the $(1+1)$~EA, is inefficient in tracking the optimal region in spite of being equipped
with the same mutation operator.
On the other hand, we show that a population-based 
algorithm  with a sufficiently large population can efficiently track the moving 
optimal region of any dynamic function of the class defined for any given finite search 
space.

The remainder of the paper is organised as follows. The next section 
first gives a formal description of dynamic optimisation and efficient tracking, 
then the class of dynamic functions that we consider is described with an 
example function. Next, we consider the $(1+1)$~EA and RLS on 
the  function class and provide an analysis to serve as an example how search heuristics
based on single solutions are not able to track the optimal solutions over time.
The efficiency of population-based algorithms is then explained by proving
a positive result about their performance. 
Here, we use the setting of non-elitist populations and show that, with a 
sufficient selective pressure, the ability of the population to track the 
moving optimal region is overwhelmingly high with respect to the population size. On the 
other hand, as a consequence of a fair comparison to a single-individual 
approach, the population must not be too big in order to capture the frequency 
of changes. Finally, we summarise, conclude and point out directions for future research.

The paper uses the following notation and terminology. For any
positive integer $n$, define $[n]:=\{1,2,\dots, n\}$. The natural
logarithm is denoted by $\ln(\cdot)$, and the logarithm to the base
$2$ is denoted by $\log(\cdot)$. The Hamming distance is denoted by
$\hdist(\cdot,\cdot)$ and the Iverson bracket is denoted by
$[\cdot]$. We use $\indfs{A}$ to denote the indicator function of a
set $A$, \ie $\indfs{A}(x) = 1$ if $x \in A$, and $0$ otherwise.  For
a given bitstring $x \in \{0,1\}^n$, the Hamming ball around $x$ with
radius $r$ is denoted by
$\hball_r(x):=\{y \in \{0,1\}^n \mid \hdist(x,y) \leq r\}$. The
bitstring containing $n$ one-bits and no zero-bits is denoted $1^n$.
An event is said to occur with overwhelmingly high probability (w.\,o.\,p.)
with respect to a parameter $n$, if the probability of the event is
bounded from below by $1-e^{-\Omega(n)}$.


\makeatletter{}
\section{A general class of dynamic functions}\label{sec:func}

Before defining the class of $(\kappa,\rho)$-stable dynamic functions
which will be studied in this paper,
we first formalise our notion of dynamic optimisation, and we define what we mean when saying that
a dynamic search heuristic tracks a moving optimal
region efficiently.

\makeatletter{}
\subsection{A formal description of dynamic optimisation}\label{sec:dyn-opt}

We focus on optimisation of pseudo-Boolean functions with
discrete-time dynamics, as formalised below. Note that our 
formalisation can be generalised to any finite search space $\mathcal{X}$, \eg replacing $\{0,1\}^n$ with $\mathcal{X}$, 
and our results for population-based algorithms also hold for this
generalisation.

\begin{definition}\label{def:dyn-func}
  A dynamic function $F$ is a random sequence of functions
  $(f_t)_{t\in\mathbb{N}},$ where $f_t:\{0,1\}^n\rightarrow\mathbb{R}$
  for all $t\in\mathbb{N}$.
  The optimal regions associated with $F$ is the sequence
  $(\opt{t})_{t\in\mathbb{N}}$, where $\opt{t}=\arg\max_x f_t(x)$.
\end{definition}
The perhaps simplest, non-trivial example of a dynamic function is a
periodic function which deterministically alternates between two
functions, say $g_1$ and $g_2$, such that $f_{2i-1}=g_1$ and
$f_{2i}=g_2$ for all $i\in\mathbb{N}$. We will consider more complex
dynamic functions, where the sequence of functions is random and
non-periodic. Although the sequence of functions in a dynamic function
is random, each individual function is deterministic, i.e., we do not
consider dynamic optimisation with noisy functions.

In this paper, we do not make any assumption about the changes of the 
function and the speed of the algorithm. It has been pointed out that it is important
to consider the relationship between the speed of the execution 
platform where the algorithm runs and the speed of change of the function because
this has significant influence on the performance \cite{bib:Jansen2014}. Almost
all studies assume that the  function cannot change within one generation of the
algorithm. The only exception we are aware of is a paper by Branke and Wang 
\cite{BrankeWang2003} who analyse a $(1, 2)$ EA. We follow this idea but consider
a much broader class of algorithms.

When applying a search heuristic to a dynamic function, we therefore
have to consider two time lines: the first is associated with the
evolution of the dynamic function, and the second is associated with
the search points generated by the heuristic. Following the convention
from black-box complexity \cite{Droste2006BlackBox}, we assume that
the function evaluations are the most expensive operations, for sake
of the analysis becoming the basic time steps of an algorithm.  The
time consumed by all other operations, such as sampling an individual
or applying a mutation operator, is assumed to be negligible.  We connect
the two time-lines by assuming that every time the heuristic queries a
search point, the time-line of the dynamic function increases by
one. We allow dynamic search heuristics some flexibility in that
search points can be queried not only with respect to the most recent
function $f_t$, but also with respect to past functions. For example,
the individuals in a population can be evaluated with respect to one
particular time. We also assume that the decisions made by the search
heuristic does not influence the dynamic of the function. The dynamic
optimisation-scenario we have described is summarised in the following
definition.

\begin{definition}
  A dynamic search heuristic is an algorithm which given a search
  history $\left((x_j,i_j,f_{i_j}(x_j)\right)_{j\in[t-1]}$ of $t-1$
  elements in $\{0,1\}^n\times\mathbb{N}\times\mathbb{R}$, selects a
  search point $x_t\in\{0,1\}^n$ and an evaluation time $i_t\in[t]$,
  and evaluates $f_{i_t}(x_t)$.
\end{definition}

An element $(x_t, i_t, f_{i_t}(x_t))$ in a search history describes
the search point $x_t$ queried by the algorithm in step $t$, the time
point $i_t\leq t$ with which the search point is evaluated, and the
corresponding function value $f_{i_t}(x_t)$. We can now 
formalise the notion of 
\emph{efficient tracking} of optima.

\begin{definition}\label{def:eff-tracking}
  A search heuristic is said to \emph{efficiently track the optima} of a
  dynamic function $F$   if there exist
  $t_0,\ell\in\poly(n)$ and constants $c,c'>0$ such that
  \begin{align*}
    \min_{t_0<t<e^{cn}}\prob{ \sum_{i=t}^{t+\ell}\indf{x_i\in\opt{i}} \geq c'\ell}\geq 1-e^{-\Omega(n)},
  \end{align*}
  where $(x_t)_{t\geq 0}$ is the sequence of search points queried
  by the heuristic, and $(\opt{t})_{t\geq 0}$ is the sequence of
  optimal search points of function $F$.
\end{definition}

Informally, Definition \ref{def:eff-tracking} means that the algorithm
queries optimal search points frequently. More precisely, within every
sub-interval of length $\ell$ within the exponentially long time
interval from $t_0$ to $e^{cn}$, a constant fraction of the queried search
points are optimal.
  Note that the optimality of a search point is defined with
  respect to the query time, and regardless of the function with
  which the algorithm evaluates the search point. The constraint
  $\ell\in\poly(n)$ on the length of sub-intervals guarantees that the
  time between generation of two optimal search points is bounded from
  above by a
  polynomial. 
It is clear from the definition that an algorithm is inefficient if with
a sufficiently high probability, \eg at least constant, it loses track of the 
optimal region and does not recover it within a polynomial number of steps.

\makeatletter{}
\subsection{A class of stable dynamic functions}\label{sec:stable-dyn-func}

The class of $(\kappa,\rho)$-stable dynamic functions with respect to
a variation operator is defined as follows.

\begin{definition}\label{def:stable-dyn-func}
Let $\phi \colon \{0,1\}^n \rightarrow \{0,1\}^n$ 
be any \emph{unary} variation operator, 
and 
$\kappa\in\mathbb{N}$, $\rho\in(0,1)$.
If there exist
  constants $c,c'>0$ such that with probability at least $1 - e^{-c'\kappa}$,
  the optimal regions $(\opt{t})_{t\in\mathbb{N}}$ of a 
  function $F$ satisfy 
  for all time points $t$ and $t'$ with
  $0 < t < t' \leq t+\kappa<e^{c\kappa}$,
  and for all search points $x\in\opt{t}$,
  \begin{align*}
    \prob{\phi(x) \in \opt{t'}} \geq \rho
  \end{align*}
then $F$ is called $(\kappa,\rho)$-stable with respect to $\phi$.
\end{definition}

  Definition \ref{def:stable-dyn-func} covers a large class of
  dynamic optimisation functions for any given pair of parameters
  $(\kappa,\rho)$.  The optimal regions over time can take many
  shapes, including disconnected pieces over $\{0,1\}^n$
  as long as the distances between them and the cardinality of the 
  intersections allow the probabilistic condition to hold.  Figure
  \ref{fig:dyn-func-shape} illustrates the required condition.
  
  Given an operator $\phi$, 
  we focus on the class of $(cn,\rho)$-stable functions 
  where $c$ and $\rho$ are positive constants. We will show that a population-based
  algorithm with a sufficiently large population and a sufficiently
  strong selection pressure can track the optimal region of any function in
  the class efficiently. The next section gives an example 
  function of the class for $\phi$ being bitwise mutation
  and explains how it fits within the
  framework of $(cn,\rho)$-stable function. 
        We will then use the example function to argue that
  algorithms that base their search on a single individual, such as 
  the (1+1)~EA, can be inefficient.
\begin{figure}
\begin{center}
\includegraphics[width=4cm]{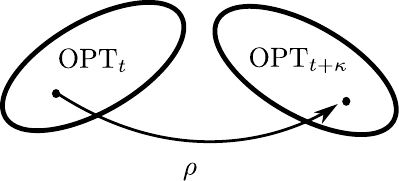}
\end{center}
\caption{Illustration of a $(\kappa,\rho)$-stable dynamic function,
         in which any search point in the optimal region of time $t$ 
         can be mutated into the optimal region of time $t+\kappa$
         with probability at least $\rho$.}
\label{fig:dyn-func-shape}
\end{figure}

\makeatletter{}
\subsection{Example of a stable pseudo-Boolean function}\label{sec:example-stable-func}

We consider the Moving Hamming Ball function as described in 
\cite{bib:Dang2015}. The static version of the function has the following form.

\begin{definition}\label{def:hb-func} 
The Hamming Ball function around a target bitstring $x^*$ and a radius $r$ is 
defined as,
\begin{align*}
  \hbfunc^{r,x^*}(x) 
    &= \begin{cases}
       1 & \text { if } \hdist(x, x^*) \leq r, \\
       0 & \text { otherwise. }
       \end{cases}
\end{align*}
\end{definition}

It suffices to change $x^*$ in sequential steps to create a dynamic 
version from the static one. We use the following dynamic setting for the 
function: the points in time when the target $x^*$ is
changed are 
determined by a sequence of random variables drawn from a Poisson distribution.

\begin{definition}\label{def:mhb-func} 
Let $(X_i)_{i\in\mathbb{N}}$ be a sequence of random variables independently sampled 
from a Poisson distribution with parameter $\theta$, $X_i \sim \pois(\theta)$, 
$\ell$ be some integer in $[n]$, and 
$(x^*_i)_{i\in\mathbb{N}}$ be a sequence of bitstrings generated by 
\begin{align*}
  x^*_i 
    &= \begin{cases}
           1^n & \text{ if } i = 0, \\
           \sim \unif(\{y \mid \hdist(x^*_{i-1},y) = \ell\}) & \text{ otherwise. }
       \end{cases}
\end{align*}
The Moving Hamming Ball ($\mhbfunc$) function with parameters $r$,
$\ell$, and $\theta$ is defined as
\begin{align*}
  \mhbfunc^{r,\ell,\theta}_t(x) 
    =     \hbfunc^{r,x^*(t)}(x) \\
  \quad \text{where }
       x^*(t) = x^*_{\max\{j \mid \sum_{i=1}^{j} X_i \leq t\}}.
\end{align*}
\end{definition}

The $\mhbfunc$ function fits within the stability framework of
Definition~\ref{def:stable-dyn-func} with respect to the bitwise
mutation operator $\pmutea$. This variation operator, which has a
parameter $\chi\in[0,n]$, flips each position in the bitstring
independently with probability $\chi/n$. Hence, the probability of
mutating a bitstring $x\in \{0,1\}^n$ into a bitstring $y\in\{0,1\}^n$
is
\begin{align*}
  \prob{y=\pmutea(x)} = \left(\frac{\chi}{n}\right)^{H(x,y)}\left(1-\frac{\chi}{n}\right)^{n-H(x,y)}.
\end{align*}

\begin{lemma}\label{lem:mhb-func-stable}
For all positive constants $d$, $\chi$ and $\varepsilon$, the function
$\mhbfunc^{r,\ell,\theta}$ is $(\frac{\theta}{1+d}, 
\left(\frac{r\chi}{n\ell}\right)^{\ell} e^{-(1+\varepsilon)\chi})$-stable with 
respect to the bitwise mutation operator $\pmutea$ with parameter $\chi$.
\end{lemma}
\begin{proof}
For any given time $t$, let $X$ be the random variable associated with 
the number of time steps in the future that the target bitstring will be 
changed. If we pick $\kappa := \theta/(1+d)$, then it is clear that within the 
next $\kappa$ time steps, there will be more than one change of the target 
bitstring if and only if $X \leq \kappa$. It follows from Lemmas
\ref{lem:poisdist} and \ref{lem:lnbound} that
\begin{align*}
  \prob{X \leq \kappa} 
    &\leq e^{-\theta}\left(\frac{e\theta}{\kappa}\right)^{\kappa}
     =    e^{-(1+d)\kappa}\left((1+d)e\right)^{\kappa} \\
    &=    e^{-d\kappa}(1+d)^{\kappa}
     \leq \exp\left(-d\kappa + \kappa\cdot \frac{d}{2}\cdot\frac{d+2}{d+1}\right) \\
    &=    \exp\left(-\frac{\kappa}{2}\left(\frac{d^2}{d+1}\right)\right).
\end{align*}

It suffices to pick the constant $\upsilon := \frac{d^2}{2(d+1)}$ so 
that with a probability of at least $1 - e^{-\upsilon\kappa}$, there is at most 
one change to the target function within the next $\kappa$ time steps. Under 
that condition, it holds for all $t' \in [t, t+\kappa]$ and for all $x \in 
\hball_r(x(t)) =: \opt{t}$, that $\hdist(x, x(t')) = r + \ell'$ for some $\ell' \in 
\{0\} \cup [\ell]$.

In the case that $\ell'=0$, \eg the target does not move or it moves 
closer to $x$,  it suffices to not flip any of the $n - r$ bits. For any 
constant $\varepsilon$, it holds for all $n \geq (1+1/\varepsilon)\chi$ that
\begin{align*}
  \prob{\pmutea(x) \in \hball_r(x(t'))\mid \ell'=0}
    &\geq \left(1 - \frac{\chi}{n}\right)^{n-r}
     \geq \left(1 - \frac{\chi}{n}\right)^{\left(\frac{n}{\chi}-1\right)\chi\left(1 + \frac{\chi}{n - \chi}\right)} \\
    &\geq e^{-(1+\varepsilon)\chi}.
\end{align*}

In the case that $\ell'>0$, it suffices to recover the $\ell'$ bits 
among the $r + \ell'$ mismatched ones, so
\begin{align*}
  \prob{\pmutea(x) \in \hball_r(x(t'))\mid \ell'>0}
    &\geq \binom{r + \ell'}{\ell'}\left(\frac{\chi}{n}\right)^{\ell'} \left(1 - \frac{\chi}{n}\right)^{n - \ell'} \\
    &\geq \left(\frac{r + \ell'}{\ell'}\right)^{\ell'}\left(\frac{\chi}{n}\right)^{\ell'} e^{-(1+\varepsilon)\chi} \\
    &>    \left(\frac{r\chi}{\ell' n}\right)^{\ell'}e^{-(1+\varepsilon)\chi}
     \geq \left(\frac{r\chi}{\ell n}\right)^{\ell}e^{-(1+\varepsilon)\chi}.
\end{align*}

Note that $\opt{t'} := \hball_r(x(t'))$, hence 
\begin{align*}
  \prob{\pmutea(x) \in \opt{t'}} 
    \geq \left(\frac{r\chi}{\ell n}\right)^{\ell}e^{-(1+\varepsilon)\chi} =: \rho
\end{align*}
and $\mhbfunc^{r,\ell,\theta}$ is $(\kappa,\rho)$-stable with respect to 
$\pmutea$.
\end{proof}

It is not difficult to see that the stability condition of
  the function class still holds with the following relaxations:
\begin{itemize}
  \item the fitness of the solutions inside the Hamming ball changes when the
target string moves,
  \item the fitness of the solutions outside the current Hamming Ball can be 
distributed differently, as long as they are less than the current optimal 
fitness,
  \item the moving step $\ell$ is relaxed to be sampled from any discrete 
distribution over $[\ell]$ in each change of the target bitstring.
\end{itemize}
We will not consider these relaxations
as they are not required to
distinguish between the effectiveness of single-individual and
population-based evolutionary algorithms.

\makeatletter{}
\section{Algorithms}\label{sec:algo}

We will compare the performance of population-based and
single-individual based evolutionary algorithms. In this section we
first define these classes of algorithms.

We are considering dynamic optimisation problems and, as mentioned in the
introduction and discussed by Jansen and Zarges \cite{bib:Jansen2014}, it is
important to clarify how the algorithms deal with change of the fitness functions,
in particular if this happens during one generation.
In this paper, we consider algorithms that make use of consistent 
comparisons when applying on a dynamic function: when an 
algorithm has to make fitness comparisons on a set of solutions, it will 
first make a \emph{static copy} of the dynamic function and evaluate 
the solutions on this copy. This approach corresponds to an implementation 
where the necessary data to evaluate the optimisation function is collected 
before evaluating a set of solutions. Meanwhile the real function 
may have changed more or less depending on the number of 
solutions in the set.

We first consider 
the single-individual approach described in
Algorithm~\ref{algo:onesol}.  The algorithm keeps a current
  search point $x$. In each iteration, it produces a new candidate
  solution $x'$, and compares it with the current search point using
  the same function. Hence, static copies of the dynamic function are
  made in every two time steps. This corresponds to a frequent update
  of the dynamic function. We let $\pmut$ be the
  bitwise mutation operator $\pmutea$ described in
  Section~\ref{sec:example-stable-func}, and obtain the well-known
  (1+1)~EA \cite{DJW02Analysis}. However, the result can be easily 
  generalised to other mutation operators, such as the single-bit flip
  operator used in the RLS algorithm.

\begin{algorithm}
  \caption{Single-individual Algorithm (Dyn.\ Opt.)}
  \begin{algorithmic}[1]
    \REQUIRE ~\\ 
             finite search space $\mathcal{X}$, \\ 
             dynamic function $F = (f_t)_{t \in \mathbb{N}}$, \\ 
             initial solution $x_0 \in \opt{0}$.
    \FOR{$\tau=0,1,2,\dots$ until termination condition met}
    \STATE $x' = \pmut(x_{\tau})$.
    \STATE $v_1 = f_{2\tau}(x')$.
    \STATE $v_2 = f_{2\tau}(x_{\tau})$.
    \IF{$v_1 \geq v_2$}
      \STATE $x_{\tau + 1} = x'$.
     \ELSE
     \STATE{$x_{\tau + 1} = x_{\tau}$.}
   \ENDIF
    \ENDFOR
  \end{algorithmic}
  \label{algo:onesol}
\end{algorithm}

We are mostly interested in the influence of the population
size, designated by the parameter $\lambda$, on the ability
of a population-based algorithm to track the moving optimal 
region. We focus on the non-elitist setting as described in
Algorithm~\ref{algo:popbased}. The algorithm uses a unary 
variation operator denoted by $\pmut$, no crossover operator, 
and an unspecified selection mechanism $\psel$. The selection mechanism 
is any random operator $\psel$ that given a population $P$ and access to
a fitness function returns one of the individuals 
in $P$.  By specifying different $\psel$ and $\pmut$, Algorithm~\ref{algo:popbased} 
can instantiate a large number of population-based search heuristics, 
such as the ($\mu,\lambda$)~EA.  The number of search points $\lambda$
produced in each round is the only parameter that appears in the 
description of Algorithm~\ref{algo:popbased}. The ($\mu$,$\lambda$)~EA fits within this framework
by making sure that the selection in line~\ref{step:selection} only takes into account the $\mu$ best
of the $\lambda$ search points created in the last round.

The algorithm maintains a population $P_\tau$ of $\lambda$ individuals
which during one generation (steps \ref{step:genstart}--\ref{step:genend}) is 
replaced by a newly created population $P_{\tau+1}$ of the same size.  As for 
the (1+1)~EA, we assume that the initial population $P_0$ is contained
in the first optimal region $\opt{0}$. Each individual in the next population
$P_{\tau+1}$ is created by first making a copy $x$ of one parent 
individual which is selected from the current population (step \ref{step:selection}, 
selection), then modifying the copy using $\pmut$ operator (step 
\ref{step:mutation}, mutation).

When selecting individuals, the algorithm must take into account that
multiple changes to the fitness function can occur during one
generation.  Here, we assume that the algorithm makes a static copy of
the fitness function $f_{\tau\lambda}$ at the beginning of each
generation, i.e. at time $\tau\lambda$. The selection mechanism
$\psel$ compares all individuals in a generation using the static
copy. Note that if the population size $\lambda$ is too large with
respect to the problem parameter $\theta$ (which controls the
frequency of change of the dynamic function), then the optimal region
may change several times between two consecutive generations. Hence,
the population size should not be too large. However, we will show in
the next section that a sufficiently large population is also
essential to keep the population within the optimal region.

The result for populations will be first shown for any finite search 
space and any mutation operator $\pmut$ because the class of dynamic function is defined
with respect to the operator $\pmut$. Then we will use $\pmutea$ over $\{0,1\}^n$ as 
a specific example.

\begin{algorithm}   \caption{Non-elitist EAs (Dyn. Opt.)}
  \begin{algorithmic}[1]
    \REQUIRE ~\\
             finite search space $\mathcal{X}$, \\ 
             dynamic function $F = (f_t)_{t \in \mathbb{N}}$, \\ 
             initial population $P_0\subset\opt{0}$.
    \FOR{$\tau=0,1,2,\dots$ until termination condition met}
    \STATE evaluate solutions of $P_{\tau}$ with $f_{\tau\lambda}(\cdot)$.\label{step:genstart}
    \FOR{$i=1,2,\dots,\lambda$}
    \STATE $x = \psel(P_{\tau})$.\label{step:selection}
    \STATE $P_{\tau+1}(i) = \pmut(x)$.\label{step:mutation}
    \ENDFOR\label{step:genend}
    \ENDFOR
  \end{algorithmic}
  \label{algo:popbased}
\end{algorithm}

Although Algorithm~\ref{algo:popbased} can use any selection mechanism $\psel$,
we are looking for choices of $\psel$ that allows the algorithm to
track optima efficiently.
Formally, $\psel$ applied on finite populations of size $\lambda$ is 
represented by the transition matrix $\psel\colon [\lambda] \times 
\mathcal{X}^\lambda \rightarrow [0,1]$, where $\psel(i \mid P)$ represents the 
probability of selecting individual $P(i)$, \ie the $i$-th individual, of $P$. 
We also write $x=\psel(P)$, \eg in the algorithm, to express that $x$ 
is sampled from the distribution over $P$ given by $\psel(\cdot,P)$.
We use $x_{(i)}$ to denote the $i^{\mathrm{th}}$ best individual of $P$, or the 
so-called $(i/\lambda)$-ranked individual. Similar to 
\cite{bib:Dang2016,bib:Lehre2011}, we characterise $\psel$ by the cumulative 
selection probability. 

\begin{definition}[\cite{bib:Lehre2011}]\label{def:cumulsel}
Given a fitness function $f\colon \mathcal{X} \rightarrow \Real$, the 
\emph{cumulative selection probability} $\beta$ associated with selection 
mechanism $\psel$ is defined on $f$ for all $\gamma\in (0,1]$ and 
a $P\in\mathcal{X}^\lambda$ by
  \begin{align*}
    \beta(\gamma,P) := \sum_{i\in [\lambda]} \psel(i \mid P) \cdot \left[f(P(i)) \geq f(x_{(\lceil \gamma \lambda \rceil)})
    \right].
  \end{align*}
\end{definition}

Informally, $\beta(\gamma,P)$ is the probability of selecting an
individual with fitness at least as good as that of the
$\gamma$-ranked individual, assuming that $P$ is sorted according to
fitness values. We are interested in a lower bound function of
$\beta(\gamma,P)$.
Most often this lower bound is independent of $P$, in which case we simply 
write it as a function of $\gamma$ only, i.\,e.\ as $\beta(\gamma)$.

\makeatletter{}
\section{Performance analysis}\label{sec:runtime}

In this section, we first show that single-individual approaches are
inefficient in tracking moving optima on at least one example function of the 
class, precisely on $\mhbfunc^{r,\ell,\theta}$. Then we prove a general 
result that an appropriately parameterised population-based algorithms
can efficiently  track the moving optima of any function in the class.

\makeatletter{}
\subsection{Inefficiency of a single individual}

In this section, we will show that the $(1+1)$~EA spends an exponential 
fraction of its time outside the optimal region of a $\mhbfunc^{bn,\ell,cn}$
function, for a sufficiently small constant $b$, any constant $c>0$ and any 
$\ell\geq 1$. That is to say the algorithm is inefficient even in tracking such a 
stable function.

To prove such a result, we have to analyse the behaviour of the
algorithm both inside and outside the moving Hamming ball: We assume
that the algorithm starts at the center of the first optimal region
and show that after some initial time, whenever the center of the ball
moves, there is a constant probability that the $(1+1)$~EA will
memorise a search point outside of the new ball; Whenever the
algorithm is outside of the optimal region there is also a constant
probability that the memorised search point will drift away from the
optimal region (eventually get lost), before an optimal solution
inside the new ball is discovered. Since the changes to the function
happens within an expected polynomial number steps, we can conclude
that with a high probability, the $(1+1)$~EA only spends a polynomial
number of time steps inside the moving Hamming ball.

We start with the first argument, the behaviour of the algorithm
inside the Hamming ball. We notice that the changes induced by the
dynamics of the fitness function strongly drag the target away from
the current memorised search point, however this does not happen in
every iteration. In every iteration, the changes by mutation drive the
memorised solution away from the center of the current Hamming ball,
but the elitist selection also keeps the memorised solution from
falling outside. We have the following analysis of the drift.

We consider the process $(X_t)_{t\in \mathbb{N}}$, where $X_t$ is the
Hamming distance to the border of the optimal region of $\mhbfunc^{bn,\ell,cn}$ 
at time $t$, \ie $X_t = r - \hdist(x^*(t), x_\tau)$. The process starts with 
$X_t = r$, \eg exactly at the center of the Hamming ball. Given $X_t = i$, define 
$\Delta(i) := X_{t} - X_{t+1}$, then $\expect{\Delta(i) \mid X_{t} = i}$ is the 
drift towards the border at time $t$ and where $X_t = i$.

First of all, the dynamic now only kicks in  every 
$cn$ time steps in expectation. Also, the contributing drift is positive. 
For example, if the dynamic kicks in, let $Z$ be the number of bits being corrected 
by the dynamic, then we have $Z \sim \hgeo(n,n-i,\ell)$, and the contributing drift is 
$\expect{\ell - 2Z} = \ell(1 - 2i/n) > 0$ for any $r/n<1/2$.

We now compute the drift by mutation at time $t$ and where $X_t = i > 0$. 
Let $X$ and $Y$ be the number of bits being corrected and being messed up respectively 
by the mutation, so $X \sim \bin(r - i,\chi/n)$, $Y \sim \bin(n-(r - i),\chi/n)$ and 
the two variables are independent. Note that for all integers $X \geq 0$, $Y \geq 0$ 
and $i \geq 1$, it holds \begin{align*}
  \Delta(i) 
    &= (Y - X) \cdot \indf{Y - X \leq i}      =    Y \cdot \indf{Y \leq i + X}           - X \cdot \indf{X \geq Y - i} \\     &\geq Y \cdot \indf{Y \leq 1}           - X  
     =    \indf{Y = 1}           - X 
     =: \Delta_1(i).
\end{align*}

Thus for $i > 0$, $\Delta(i)$ stochastically dominates $\Delta_1(i)$ and 
we also have
\begin{align*}
  \expect{\Delta_1(i) \mid X_t = i}  
    &=    \expect{ \indf{Y = 1} }           - \expect{X} \\
    &=    \binom{n - (r - i)}{1}\left(\frac{\chi}{n}\right)\left(1 - \frac{\chi}{n}\right)^{n-(r - i) - 1} - \frac{i\chi}{n} \\
    &\geq \chi\left(\left(1-\frac{r - i}{n}\right)e^{-(1+\varepsilon)\chi} - \frac{r - i}{n}\right) \\
    &>    \chi\left(\left(1-\frac{r}{n}\right)e^{-(1+\varepsilon)\chi} - \frac{r}{n}\right) \\
    &=    \chi\left(\left(1-b\right)e^{-(1+\varepsilon)\chi} - b\right)
\end{align*}
for any constant $\varepsilon>0$. Therefore, with any constant $b < 1/(1+ 
2 e^\chi)$, we have that $b \leq (1 - b)e^{-(1+\varepsilon)\chi}/2$. Hence, for 
$i > 0$ we have at least a constant drift away from the 
center
\begin{align*}
  \expect{\Delta_1(i) \mid X_t = i} 
     >    \left(\frac{\chi}{2\cdot e^{(1+\varepsilon)\chi}}\right)\left(1 - b\right) =: \delta.
\end{align*}

The only position where we have a drift toward the center is the one at the 
border, \eg $X_t = 0$. However, this is not a strong drift. When the target 
bitstring does not move in the next iteration, we have $- \Delta(0) = 
(X - Y)\cdot \indf{Y - X \leq 0} \leq X \cdot \indf{X \geq Y} \leq X$, then the negative drift 
is no more than
\begin{align*}
  \expect{ X } 
     = \frac{r \chi}{n}  =: \eta.
\end{align*}

In summary, we get the drift by mutation: 
\begin{align}
  \expect{\Delta(i) \cdot \indf{X_t > 0}           \mid X_t = i} 
    &\geq \delta 
          \cdot 
          \indf{X_t > 0}           \label{eq:delta} \\
  \expect{\Delta(i) \cdot \indf{X_t = 0}           \mid X_t = i} 
    &\geq -\eta 
           \cdot 
           \indf{X_t = 0}            \label{eq:eta}
\end{align}
It is then suggested that the equilibrium state of the memorised search point is around the border. 
Furthermore, we can quantify the expected fraction of time that the search point 
is found at the border, using the following tool.

\begin{lemma}\label{lem:occupancy}
  Given a stochastic process $(X_t)_{t\geq 0}$ over a state space $\mathbb{N}$, 
  and two constants $\eta,\delta\in\mathbb{R}_+$ such that
  \begin{itemize}
  \item $\expect{X_{t+1} \cdot \indf{X_t=0}\mid X_t}\leq \eta \cdot \indf{X_t=0}$, and
  \item $\expect{X_{t+1} \cdot \indf{X_t> 0}\mid X_t}\leq (X_t-\delta) \cdot \indf{X_t> 0}$,
  \end{itemize}
  then for all $t\geq 1$
  \begin{align*}
    \sum_{i=0}^{t-1}\prob{X_t=0} \geq 
    \frac{\delta t-X_0}{\delta+\eta}.
  \end{align*}  
\end{lemma}
\begin{proof}
  Define $p_i=\prob{X_i=0}$. For all $t\geq 1$, it holds
  \begin{align*}
    \expect{X_t} 
      & = \expect{\indf{X_{t-1}=0} \cdot X_{t}}+\expect{\indf{X_{t-1}>0} \cdot X_{t}}\\
      & = \expect{\expect{\indf{X_{t-1}=0} \cdot X_{t}\mid X_{t-1}}}+\\
      &   \quad\quad \expect{\expect{\indf{X_{t-1}>0} \cdot X_{t}\mid X_{t-1}}}\\
      & \leq \expect{\eta \cdot \indf{X_{t-1}=0} }+
        \expect{(X_{t-1}-\delta) \cdot \indf{X_{t-1}>0}}\\
      & = \eta p_{t-1}-\delta (1-p_{t-1})+\expect{X_{t-1} \cdot \indf{X_{t-1}>0} }\\
      & = \eta p_{t-1}-\delta (1-p_{t-1})+\expect{X_{t-1}}.       
  \end{align*}
  
  It follows that 
  \begin{align*}
    \expect{X_t\mid X_0} 
    \leq X_0-t\delta+(\delta+\eta)\sum_{i=0}^{t-1}p_t.
  \end{align*}
  
  Finally, since $\expect{X_t\mid X_0}\geq 0$
  \begin{gather*} 
    \sum_{i=0}^{t-1}p_t \geq \frac{t\delta-X_0}{\delta+\eta}. \qedhere
  \end{gather*}
\end{proof}

The following lemma considers non-negative, integer-valued stochastic
processes with positive drift at most $\eta$ in state 0, and negative
drift at least $\delta$ elsewhere. It provides a lower bound on the
probability of such a process being in state 0 after some time.

\begin{lemma}\label{lem:occupancy-independent-time}
  Let $(X_t)_{t\geq 0}$ be any stochastic process with support in
  $\{0\}\cup [r]$ for some fixed $r\in\mathbb{N}$, 
  which satisfies the properties of
  Lemma~\ref{lem:occupancy} for some $\delta,\eta\in\mathbb{R}_+$.
  Then for any random variable
  $T_1\geq \lceil 2r/\delta\rceil$ which is independent of
  $(X_t)_{t\geq 0}$, it holds
  \begin{align*}
    \prob{X_{T_1}=0}\geq \frac{\delta}{2(\delta+\eta)}.
  \end{align*}
\end{lemma}
\begin{proof}
  Choose $t:= \lceil 2r/\delta\rceil$, and define
  $Y_i:=X_{T_0+i}$ where $T_0:=T_1-T$ and
  $T\sim\unif( \{0\}\cup [t-1])$, i.e., we consider the 
  process $X_t$ from a random starting point $T_0\geq 0$.
  Due to independence between $T_1, T$,  and $(X_t)_{t\geq 0}$, 
  we have
  \begin{align*}
    \prob{X_{T_1}=0} 
      & = \sum_{i=0}^{t-1}\prob{Y_i=0 \wedge T_0+i=T_1}\\
      & = \sum_{i=0}^{t-1}\prob{Y_i=0}\prob{T = i}\\
      & = \sum_{i=0}^{t-1}\frac{1}{t}\prob{Y_i=0}.
  \end{align*}
  Lemma~\ref{lem:occupancy} applied to $(Y_t)_{t\geq 0}$ 
  now implies
  \begin{gather*}
    \sum_{i=0}^{t-1}\frac{1}{t}\prob{Y_i=0} 
    \geq \frac{\delta-Y_0/t}{\delta+\eta}
    \geq \frac{\delta}{2(\delta+\eta)}. \qedhere
  \end{gather*}
\end{proof}

We now show that once the $(1+1)$~EA has lost track of the
  optimal region it will take a long time to recover. We assume that
  the objective function is $\mhbfunc^{bn,\ell,cn}$ with radius 
  $r = bn \ll n/2$, \ie $b \leq (1/2) - \kappa$ for some constant 
  $\kappa > 0$ (note that $b$ can depend on $n$). The first
  step in this proof is to show that with not too small probability
  the $(1+1)$~EA ends up far away (more specifically, in a linear
  distance) from the Hamming ball before recovering it.

\begin{lemma}
\label{lemma-1+1-leaving}
    Given $o \in \{0,1\}^n$, let $(x_t)_{t \geq 0}$ be a sequence of random 
    bit strings such that $x_t = \pmutea(x_{t-1})$ and $x_0 \in \hball_{r+1}(o)$ for 
    some $r = bn \ll n/2$, \ie $0< b \leq 1/2 -\kappa$ for some $\kappa>0$.     
    For any $d \in \mathbb{N}_{+}$, define $T_{r, d} := \inf\left\{ t \mid 
    \hdist(x_t,o) \leq r\right.$ or $\left.\hdist(x_t,o)\geq r+d \right\}$.
    It holds that $\Prob{\hdist(x_{T_{r, d}},o) \leq r} = \bigO{\max\{r, \log 
    n\}/n}$ where $d = \varepsilon n$ for a not too large constant $\varepsilon > 
    0$. 

\end{lemma}

\begin{proof}
We begin with considering another random sequence $y_0, y_1, y_2, y_3, \dots$ where for each $t \in \mathbb{N}$ the point $y_t$ is created by flipping one randomly selected bit in $y_{t-1}$. Let $T'_{r, d}$ be defined as $T_{r, d}$ but with respect to $y_t$ instead of $x_t$.

Let $p_x := \Prob{\hamming{y_{T'_{r, d}}}{o} \leq r \mid \hamming{y_0}{o} = x}$, i.\,e., the probability to enter the Hamming ball before reaching distance $d$ given the process is started with Hamming distance $x$. Note that, for symmetry reasons, $p_x$ is well defined, i.\,e., the probability does only depend on the Hamming distance $x$ and not the specific choice of $y_0$.

By definition of $T'_{r, d}$ we have $p_x = 1$ for $x \leq r$ and $p_x = 0$ for $x \geq r+d$. For all other values of $x$, i.\,e., for $x \in \{r+1, r+2, \dots, n-r-1\}$ we have 
\begin{equation*}
	p_x = \left(\frac{n-x}{n}\right) p_{x+1} + \left(\frac{x}{n}\right) p_{x-1}
\end{equation*}
by definition of the sequence $y_t$ because with probability $(n-x)/n$ the Hamming distance to the centre of the Hamming ball $o$ is increased by 1 and with the remaining probability $x/n$ it is decreased by 1. If we pessimistically assume that the probability to move towards the Hamming ball is always equal to $(d+r-1)/n$ we obtain an upper bound on $p_x$ and are in the situation of the gambler's ruin problem with initial funds $s_a=x-r$ and $s_b=d+r-x$, $p_a = (n-d-r+1)/n$, and $p_b = (d+r-1)/n$ and the probability to be ruined
\begin{equation*}
	q(r, d, x) 
	= \frac{\left(\frac{d+r-1}{n-d-r+1}\right)^{x-r} - \left(\frac{d+r-1}{n-d-r+1}\right)^{d}}{1-\left(\frac{d+r-1}{n-d-r+1}\right)^{d}}
\end{equation*}
gives 
an upper bound on the probability to  enter the Hamming ball before reaching distance $d$ when starting with Hamming distance $x$ to the centre of the Hamming ball.
We consider the probability $q(r, d, x)$ for different values of $r$, $d$ and $x$. We are interested in the results for $d = \bigTheta{n}$ and consider for this $d = \varepsilon n$ where we chose the constant $\varepsilon > 0$ such that $d+r \leq (n/2) - \delta n$ for some positive constant $\delta$. It is clear that due to the upper bound on $r$ such a constant $\varepsilon$ exists. 

It is not difficult to see that $\lim\limits_{n\to\infty} q(r, d, r+1) = \bigTheta{(r+d)/n}$. For $r = \bigTheta{n}$ this is $\bigTheta{1}$ and the best bound we can obtain. For $r = \littleo{n}$ we need to be more precise.

We begin with the case $r = \littleo{n}$ and $r = \bigOmega{\log n}$. For this setting we consider $q(r, \log n, r+1)$ and know that $\lim\limits_{n\to\infty} q(r, \log n, r+1) = \bigTheta{r/n}$ holds. Now we consider $q(r, \varepsilon n, r+\log n)$ and see that $\lim\limits_{n \to \infty} q(r, \varepsilon n, r+\log n) = \littleo{r/n}$ holds.

Finally, for the case $r = \littleomega{\log n}$, we also consider $q(r, \log n, r+1)$ and know that $\lim\limits_{n\to\infty} q(r, \log n, r+1) = \bigTheta{(\log n)/n}$ holds. Now we consider $q(r, \varepsilon n, r+\log n)$ and see that $\lim\limits_{n \to \infty} q(r, \varepsilon n, 2r) = \littleo{(\log n)/n}$ holds.

Together, we have $p_{r+1} = \bigO{\max\{r, \log n\}/n}$ for all values of $r$ and $d = \varepsilon n$. Since the sequence $y_t$ corresponds to `local mutations' this proves the claim for random local search. We generalise the statement to the (1+1)~EA in the following way. We can express a standard bit mutation as a process where first a random number $k \in \{0, 1, 2, \dots, n\}$ is chosen and then $k$ bits are selected uniformly at random to be flipped. The case $k=0$ does not flip any bit and can be ignored. The case $k=1$ is covered by the analysis for RLS. For larger $k = \bigO{\log n}$ we observe that such a step is very similar to a sequence of $k$ steps where exactly 1 bit is flipped. The difference does not change the limits we considered above. Since in one standard bit mutation $k$ bits flip with probability $\bigTheta{1/k!}$ we can ignore steps where $\littleomega{\log n}$ bits flips since they contribute too little to change the asymptotic result.
\end{proof}

\begin{theorem}
\label{the-1+1-time}
	On $\mhbfunc^{bn,\ell,cn}$ with any constants $b \in (0, 1/(1+2e))$, $c>0$ and $\ell>0$ 
	the $(1+1)$ EA with mutation rate $1/n$
	will spend only an exponentially small proportion of its time in optimal regions.
\end{theorem}
\begin{proof}
We assume the process starts inside the first Hamming ball, and 
consider the process as $(X_t)_{t\geq 0}$ as described before in Lemma 
\ref{lem:occupancy}. For the standard mutation, we have the drifts according to 
Equations \ref{eq:delta} and \ref{eq:eta} are $\delta = (1 - b)/2e$ and $\eta = 
b$. Applying Lemma \ref{lem:occupancy-independent-time} gives that after 
$r/(2\delta) = 2ebn/(1-b) = \Theta(n)$ time steps, whenever the center of the 
Hamming ball is moved, it holds that $\prob{\hdist(x^*(t),x_\tau) = r}  \geq 
\delta/(\eta + \delta) = (1-b)/(2(1 - b + 2eb))$. Conditioned on this event, the 
probability that the dynamic move $x^*(t+1)$ so that $x_\tau \notin 
\hball_r(x^*(t+1))$ is $\prob{l - 2Z \geq 1} = \prob{Z \leq \ell/2}$ where $Z 
\sim \hgeo(n,r,\ell)$, \eg the number of bit positions being corrected by the 
dynamic (see the definition of $Z$ before Lemma \ref{lem:occupancy}). Therefore, 
$\expect{Z} = \ell r/n = b\ell$ and by Markov's inequality and $b < 1/(1+2e)$ 
it holds that
\begin{align*}
  \prob{x^*(t+1),x_\tau > r} 
    &\geq \prob{\hdist(x^*(t),x_\tau) = r} \left(1 - \prob{Z > \ell/2}\right) \\
    &\geq \prob{\hdist(x^*(t),x_\tau) = r} \left(1 - 2\expect{Z}/\ell\right) \\
    &\geq \frac{(1 - b)(1 - 2b)}{2(1 - b + 2eb)} =: p_1 
     > \frac{2e - 1}{4(2e + 1)} > 0.
\end{align*} 

When the $x_\tau$ is outside of the current Hamming ball, it
  follows from Lemma \ref{lemma-1+1-leaving} that there is a
  probability of at least $p_2 = 1-\bigO{r/n} = 1-\bigO{b} > 0$ that
  the $(1+1)$~EA reaches linear Hamming distance to the Hamming ball
  before finding its way back to it. Application of the negative drift
  theorem \cite{OlivetoWitt2011} yields that the probability to find
  the way back into the optimal region within $2^{cn}$ steps is
  $\bigO{e^{-n}}$ for a sufficiently small constant $c>0$.

We have just show that in every $\Theta(n)$ time steps, whenever a 
change occurs to the target bitstring there is a probability of at least $p_1 
p_2 (1 -e^{-\Omega(n)})$ that the $(1+1)$~EA will lose track of the optimal 
region where $p_1$ and $p_2$ are constants. Applying this argument $n$ times (a 
change occurs approximately every $cn$ time steps), we conclude that with an 
overwhelmingly high probability, the $(1+1)$~EA will spend no more than 
$\bigO{n^3}$ time steps within the optimal region of the 
$\mhbfunc^{bn,\ell,cn}$ function.
\end{proof}

\makeatletter{}
\subsection{Efficiency of non-elitist, population-based algorithms}

Theorem~\ref{thm:pop-robust-long-term}, which is the main result in
this section, gives conditions under which the non-elitist,
population-based Algorithm~\ref{algo:popbased} tracks the optimal
regions of dynamic functions efficiently. We show that these
conditions can be satisfied for the moving Hamming-balls function
$\mhbfunc^{bn,\ell,cn}$ for any constant $b \in (0,1)$.

\begin{theorem}\label{thm:pop-robust-long-term}
  If there are constants $\rho,\delta>0$ and
  $\gamma_0\in(0,1)$ such that
  \begin{enumerate}
  \item $F$ is a $(\lambda,\rho)$-stable dynamic function wrt. \pmut with $\lambda=\Omega(n)$,  and
  \item $\psel$ satisfies $\beta(\gamma)\geq \gamma(1+\delta)/\rho$ for all $\gamma\in(0,\gamma_0]$,   \end{enumerate}
  then Algorithm~\ref{algo:popbased} initialised with $P_0\subset\opt{0}$
  tracks the optima of $F$ efficiently.
\end{theorem}

Condition 1 of the theorem requires that the optimal region of the
function does not move too much relatively to the variation operator
$\pmut$ during one generation.  The population size $\lambda$ is a
parameter of the algorithm which can be chosen freely. So if the
function is $(\kappa,\rho)$-stable, then the first condition can be
satisfied by setting population size $\lambda=\kappa$. Condition 2
requires that the selection mechanism \psel induces a sufficiently
high selective pressure. Note that increasingly high selective pressure is
required for decreasing values of $\rho$, where $\rho$ is the probability
of recovering the optimal search region via mutation (see
Definition~\ref{def:stable-dyn-func}).

\begin{figure}
  \centering
  \includegraphics[width=5cm]{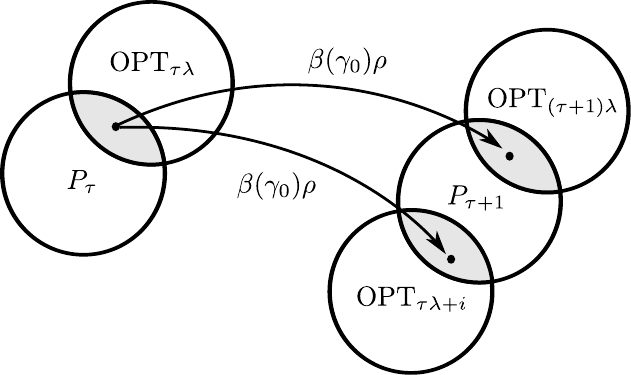}
  \caption{Illustration of Lemma~\ref{lemma:produce-nearby-opt} and
    Lemma \ref{lemma:failed-generation}.}
  \label{fig:lemma-produce-nearby-opt}
\end{figure}

The central argument in the analysis is illustrated in 
Figure~\ref{fig:lemma-produce-nearby-opt}. It follows from the
stability-assumption that any search point in $\opt{\tau\lambda}$ can
be mutated into $\opt{\tau\lambda+i}$ for any $i\in[\lambda]$ with
probability at least $\rho$. Hence, if the algorithm selects a search
point in $\opt{\tau\lambda}$ with probability $\beta(\gamma_0)$, then
the offspring belongs to $\opt{\tau\lambda+i}$ with probability at
least $\beta(\gamma_0)\rho\geq\gamma_0(1+\delta)$.  This argument is
invoked in both of the two steps of the analysis.

\begin{lemma}\label{lemma:produce-nearby-opt}
  Assume that conditions 1 and 2 of Theorem~\ref{thm:pop-robust-long-term} hold.
  Then for any $\tau\in\mathbb{N}$,  $i\in[\lambda]$, 
    if $|P_{\tau}\cap \opt{\tau\lambda}|\geq
  \gamma_0\lambda$, then any offspring in generation $\tau+1$ belongs
  to $\opt{\tau\lambda+i}$ with probability at least $\gamma_0(1+\delta)$.
\end{lemma}
\begin{proof}
  The algorithm produces an individual in $\opt{\tau\lambda+i}$
  if the algorithm selects
  an individual in $\opt{\tau\lambda}$ and mutates this individual
  into $\opt{\tau\lambda+i}$. The probability of this event is
  $\beta(\gamma_0)\rho\geq (1+\delta)\gamma_0$.
\end{proof}

Lemma~\ref{lemma:failed-generation}, which is the \emph{first step} of
the analysis, implies that in every generation $\tau\in\mathbb{N}$, a large
fraction of the population $P_\tau$ belongs to
$\opt{\tau\lambda}$. This can be shown inductively by arguing using
Lemma~\ref{lemma:produce-nearby-opt} that if many individuals in
$P_\tau$ belong to $\opt{\tau\lambda}$, then whp. many individuals in
$P_{\tau+1}$ belong to $\opt{(\tau+1)\lambda}$.  Knowing that many
individuals in $P_\tau$ belong to $\opt{\tau\lambda}$ for every
generation $\tau$ gives us some control on the dynamics of the
population. However it does not imply that the dynamic performance
measure in Definition~\ref{def:eff-tracking} is satisfied because the
individuals in $\opt{\tau\lambda}$ may not necessarily have been
optimal when they were generated.  A \emph{second step} in the
analysis is therefore required, showing that if sufficiently many
individuals in population $P_{\tau}$ belong to $\opt{\tau\lambda}$,
then many offspring in generation $\tau+1$ were optimal at the time
they were generated. This second step is contained in the proof of
Theorem~\ref{thm:pop-robust-long-term}.

\begin{lemma}\label{lemma:failed-generation}
  Assume that conditions 1 and 2 of Theorem~\ref{thm:pop-robust-long-term} hold.
  Then for any generation $\tau\in\mathbb{N}$, 
    if $|P_{\tau}\cap \opt{\tau\lambda}|\geq \gamma_0\lambda$, then
  $$\prob{|P_{\tau+1}\cap \opt{(\tau+1)\lambda}|\geq
    \gamma_0\lambda}\geq 1-e^{-\Omega(\lambda)}.$$
\end{lemma}
\begin{proof}
  By Lemma \ref{lemma:produce-nearby-opt}, any offspring in generation
  $\tau+1$ belongs to $\opt{(\tau+1)\lambda}$ independently with probability $\gamma(1+\delta)$.
  Hence, by a Chernoff bound, the probability that less than
  $\gamma_0\lambda$ offspring belongs to $\opt{(\tau+1)\lambda}$ is 
  $e^{-\Omega(\lambda)}$.
\end{proof}

We are now in position to prove the main result of this section.
\begin{proof}[Proof of Theorem~\ref{thm:pop-robust-long-term}]
  We say that generation $\tau$ \emph{fails} if
  $|P_\tau\cap \opt{\tau\lambda}|\geq \gamma_0\lambda$ and
  $|P_{\tau+1}\cap \opt{(\tau+1)\lambda}|< \gamma_0\lambda$.  By
  Lemma~\ref{lemma:failed-generation} and a union bound, the
  probability that any of the first $e^{c\lambda}/\lambda$ generations
  fails is $e^{-\Omega(\lambda)}$, assuming that $c>0$ is a
  sufficiently small constant.  By
  Lemma~\ref{fig:lemma-produce-nearby-opt} and assuming no failure,
  any individual $x_{i}$ with $\lambda<i<e^{c\lambda}$ belongs to the
  optimal region $\opt{i}$ with probability at least
  $\gamma_0(1+\delta)$.  By the definition of the algorithm, individuals
  within the same generation are produced independently. During any time 
  interval $(t,t+\lambda)$ where $t,\lambda<t<e^{c\lambda}$, 
  at least $\lambda/2$ individuals
  are produced in the same generation, and hence independently.  It
  therefore holds by a Chernoff bound that for any time interval with
  $\lambda<t<e^{c\lambda}$,
  \begin{align*}
    \prob{ \sum_{i=t}^{t+\lambda}\indf{x_i\in\opt{i}} \geq \gamma_0 \lambda/2}\geq 1-e^{-\Omega(\lambda)}.
  \end{align*}
  The theorem now follows by taking into account the failure
  probability with a union bound, and choosing the parameters $t_0=\lambda,
  \ell=\lambda,$ and $c'=\gamma_0/2$ in Definition~\ref{def:eff-tracking}.
\end{proof}

Theorem~\ref{thm:pop-robust-long-term} implies that with a sufficiently slow dynamic, \eg $\kappa 
= cn$ for any constant $c>0$, the population-based algorithm can efficiently track 
the moving optima of the function, given that $\psel$ induces a
sufficiently strong selective pressure.
We now show that given any constant $\rho \in (0, 1)$, it is possible
to parameterise many selection mechanisms so that they satisfy this requirement
on $\psel$.  The selection mechanisms are:
\begin{itemize}
  \item In \emph{$k$-tournament selection}, $k$ individuals are sampled 
uniformly at random with replacement from the population, and the fittest of 
these individuals is returned.
  \item In $(\mu,\lambda)$-\emph{selection}, parents are sampled uniformly at 
random among the fittest $\mu$ individuals in the population. 
  \item A function $\alpha:\mathbb{R}\rightarrow\mathbb{R}$ is a ranking 
function \cite{Goldberg1991Selection} if $\alpha(x)\geq 0$ for all $x\in[0,1]$, and 
$\int_0^1\alpha(x)\,\mathrm{d}x = 1$. In ranking selection with 
ranking function $\alpha$, the probability of selecting individuals ranked 
$\gamma$ or better is $\int_0^\gamma\alpha(x)\,\mathrm{d}x$. 
\emph{Linear ranking} selection uses  $\alpha(x) := \eta(1-2x)+2x$ for some 
$\eta \in (1,2]$. 
\emph{Exponential ranking} selection uses $\alpha(x):=\eta e^{\eta(1 - 
x)}/(e^\eta - 1)$  for some $\eta>0$.
\end{itemize}

The following theorem shows how these selection mechanisms can be
parameterised to satisfy the second requirement of
Theorem~\ref{thm:pop-robust-long-term}, and hence ensure that 
Algorithm~\ref{algo:popbased} tracks the moving optima of 
any $(\lambda,\rho)$-stable function with respect to the 
mutation operator $\pmut$.

\begin{theorem}\label{thm:sel-mechanism}
  For any constant $\rho\in(0,1)$, let $F$ be any $(\lambda,\rho)$-stable
  function wrt. \pmut for $\lambda=\Omega(n)$. If there is a constant
  $\delta>0$ such that Algorithm~\ref{algo:popbased} initialised with
  $P_0\subset \opt{0}$, and selection mechanism $\psel$ either
\begin{itemize}[noitemsep]
  \item $k$-tournament selection with $k \geq (1+\delta)/\rho$, 
  \item $(\mu,\lambda)$-selection with $\lambda/\mu \geq (1+\delta)/\rho$,
  \item linear ranking selection with $\eta \geq (1 + \delta)/\rho$, or
  \item exponential ranking selection with $\eta \geq (1+\delta)/\rho$,
\end{itemize}
then the algorithm tracks the optima of $F$ efficiently.
\end{theorem}
\begin{proof}
The result follows from Theorem \ref{thm:pop-robust-long-term} if we
can show that there exist constants $\delta'>0$ and $\gamma_0 \in (0,1)$ such that
$\beta(\gamma)\geq (1+\delta')\gamma/\rho$ for all $\gamma\in(0,\gamma_0]$.
The results for $k$-tournament, $(\mu,\lambda)$-selection and linear 
ranking follow from Lemmas~5, 6 and 7 from \cite{bib:Lehre2011} with $\rho$ in 
place of $p_0$. For exponential ranking, we notice that
\begin{align*}
  \beta(\gamma)
    &\geq \int_{0}^\gamma \frac{\eta e^{\eta(1 - x)}\,\mathrm{d}x}{e^\eta - 1}
     =    \left(\frac{e^\eta}{e^{\eta}-1}\right)\left(1 - \frac{1}{e^{\eta\gamma}}\right) 
     \geq 1 - \frac{1}{1 + \eta\gamma},
\end{align*}
the result then follows similarly to $k$-tournament as in the proof of Lemma~5 
in \cite{bib:Lehre2011} with $\eta$ in place of $k$ (Equations~(3) and (4) in 
that proof literally show that $\beta(\gamma) \geq 1 - 1/(1+\gamma k)$, then the 
constants $\gamma_0$ and $\delta'$ are shown to exist given the condition on $k$).
\end{proof}

Finally, we apply Theorem~\ref{thm:sel-mechanism} to show that
population-based EAs can track the optima of the example Moving Hamming Ball
function efficiently. Note that the parameter $\eta$ in linear ranking
selection can only take values in the interval $(1,2]$. The conditions
of Theorem~\ref{thm:sel-mechanism} can therefore only be satisfied if
$\rho>1/2$, i.e., the optimal regions can only change slightly. For
the last part of the paper, we therefore exclude linear
ranking selection.

\begin{corollary}\label{cor:pop-general-result-mhb}
For any constants $\delta>0, b \in (0,1)$, $c>0$, $d>0$ and $\ell \geq 1$, 
Algorithm~\ref{algo:popbased} 
with the bitwise mutation operator $\pmutea$ for $\chi=1$,
with population size $\lambda = cn/(2(1+d))$, and 
selection mechanism $\psel$ either
\begin{itemize}
  \item $k$-tournament selection with $k \geq (1+\delta)3(\ell/b)^{\ell}$,
  \item $(\mu,\lambda)$-selection with $\lambda/\mu \geq (1+\delta)3(\ell/b)^{\ell}$,
  \item exponential ranking selection with $\eta \geq (1+\delta)3(\ell/b)^{\ell}$.
\end{itemize}
can efficiently track the moving optima of $\mhbfunc^{bn,\ell,cn}$.
\end{corollary}
\begin{proof}
It follows from Lemma \ref{lem:mhb-func-stable} that for any constant 
$\varepsilon>0$, $\mhbfunc^{bn,\ell,cn}$ is $\left(\frac{cn}{1+d}, 
(b/\ell)^{\ell} e^{-(1+\varepsilon)}\right)$-stable with respect to the mutation 
operator $\pmutea$. Since $e^{-(1+\varepsilon)}>1/3$ for a sufficiently small 
$\varepsilon$, the function is also 
$\left(\frac{cn}{1+d},(1/3)(b/\ell)^{\ell}\right)$-stable. The result then 
follows by applying Theorem~\ref{thm:sel-mechanism}.
\end{proof}

\makeatletter{}
\section{Conclusion}\label{sec:concl}

This paper has considered the frequently stated intuition that
evolutionary algorithms maintaining a \emph{population} of diverse
solutions can be more resilient to dynamic changes in the objective
function than algorithms maintaining single solutions. We have
described a general class of fitness functions where population-based
evolutionary algorithms outperform single-individual evolutionary
algorithms. We have proved that for this function class,
single-individual approaches, such as the (1+1)~EA and RLS, have a
constant risk of losing the optimal solution region at any given
time. Moreover, these single-individual algorithms not only lose the
optimal region with constant probability, but are also likely to drift
away from the optimal region subsequently.

On the other hand, assuming a not too high frequency of change, we
describe sufficient conditions such that a non-elitist
population-based evolutionary algorithm will remain within the optimal
region with overwhelmingly high probability. Our analysis covers a
range of the most commonly used selection mechanisms, and we provide
appropriate parameter settings for each of them. Furthermore, the
success of the population-based evolutionary algorithm does not rely on an explicit
diversity mechanism. Our analysis gives further explanations
of how and why populations can be essential and widely used in 
dynamic optimisation.

As future work, we would like to investigate further the influence of population 
settings within this class of dynamic functions, such as elitist populations, 
the necessary condition for the population size with respect to the frequency 
and magnitude of changes and how a population could rebuild itself after losing
a few optimal solutions.

\section*{Acknowledgements}
The research leading to these results has received funding from the
European Union Seventh Framework Programme (FP7/2007-2013) under grant
agreement no 618091 (SAGE), and is based upon work from COST Action
CA15140 `Improving Applicability of Nature-Inspired Optimisation by
Joining Theory and Practice (ImAppNIO)'.


\makeatletter{}

\appendix
\section*{Appendix}

\begin{lemma}[Theorem 5.4 in \cite{bib:Mitzenmacher2005}]\label{lem:poisdist}
Let $X \sim \pois(\theta)$, then for all $\theta > x > 0$
\begin{align*}
  \prob{X \leq x} \leq e^{-\theta}\left(\frac{e\theta}{x}\right)^x.
\end{align*}
\end{lemma}

\begin{lemma}[Inequality (3) in \cite{bib:Topsoe2007}]\label{lem:lnbound}
$$\forall x \geq 0 
    \quad 1+x \leq \exp\left(\frac{x}{2}\cdot\frac{x+2}{x+1}\right).$$
\end{lemma}

\end{document}